\documentclass[letterpaper, 10 pt, conference]{ieeeconf}  
\let\labelindent\relax
\usepackage{enumitem}
\IEEEoverridecommandlockouts                              
\usepackage{subcaption}
\usepackage{etoolbox}
\makeatletter
\patchcmd{\@thm}{\trivlist}{\trivlist\itemsep=0pt \parsep=0pt}{}{}
\makeatother
\usepackage{algorithm,algpseudocode}
\renewcommand{\baselinestretch}{0.95}
\usepackage[utf8]{inputenc} 
\usepackage[T1]{fontenc}    
\usepackage{hyperref}       
\usepackage{booktabs}       
\usepackage{amsfonts}       
\usepackage{commath}
\usepackage{breqn}
\usepackage{nicefrac}       
\usepackage{microtype}      
\usepackage{xcolor}         
\usepackage{enumitem}
\usepackage{amsmath}
\usepackage{lipsum}
\usepackage{amssymb}
\usepackage{mathrsfs,bm}
\usepackage[font={small}]{caption}
\usepackage{subcaption}
\usepackage{multirow}
\usepackage{bbm}
\usepackage{arydshln}
\usepackage{etoolbox}

\makeatletter
\patchcmd{\proof}{\topsep6\p@\@plus6\p@\relax}{\topsep1pt\relax}{}{}
\makeatother
\usepackage{mathtools}
\usepackage{enumitem}
\usepackage{pict2e,picture} 
\usepackage{tikz}   
\usepackage{pifont} 
\usepackage{enumitem}
\usepackage{wrapfig}
\usepackage{bbm}
\usepackage{graphicx}
\usepackage{amsmath}
\usepackage[normalem]{ulem}
\newtheorem{lemma}{Lemma}
\newtheorem{theorem}{Theorem}
\newtheorem{theorem*}{Theorem}

\newtheorem{prop}{Proposition}

\newtheorem{assumption}{Assumption}

\usepackage{cite}
\title{\LARGE \bf
Optimistic Policy Learning under Pessimistic Adversaries with Regret and Violation Guarantees
}

\author{
Sourav Ganguly, Kartik Pandit and 
Arnob Ghosh\\Dept. of Electrical and Computer Engineering, NJIT, Newark, NJ, USA\\
\thanks{Emails: \{sg2786,ksp82,arnob.ghosh\}@njit.edu 
}
}

\begin{document}

\maketitle
\thispagestyle{empty}
\pagestyle{empty}
\begin{abstract}
Real-world decision-making systems operate in environments where state transitions depend not only on the agent's actions, but also on \textbf{exogenous factors outside its control}--competing agents, environmental disturbances, or strategic adversaries--formally, $s_{h+1} = f(s_h, a_h, \bar{a}_h)+\omega_h$ where $\bar{a}_h$ is the adversary/external action, $a_h$ is the agent's action, and $\omega_h$ is an additive noise. Ignoring such factors can yield policies that are optimal in isolation but \textbf{fail catastrophically in deployment}, particularly when safety constraints must be satisfied.
 
Standard Constrained MDP formulations assume the agent is the sole driver of state evolution, an assumption that breaks down in safety-critical settings. Existing robust RL approaches address this via distributional robustness over transition kernels, but do not explicitly model the \textbf{strategic interaction} between agent and exogenous factor, and rely on strong assumptions about divergence from a known nominal model.
 
We model the exogenous factor as an \textbf{adversarial policy} $\bar{\pi}$ that co-determines state transitions, and ask how an agent can remain both optimal and safe against such an adversary. \emph{To the best of our knowledge, this is the first work to study safety-constrained RL under explicit adversarial dynamics}. We propose \textbf{Robust Hallucinated Constrained Upper-Confidence RL} (\texttt{RHC-UCRL}), a model-based algorithm that maintains optimism over both agent and adversary policies, explicitly separating epistemic from aleatoric uncertainty. \texttt{RHC-UCRL} achieves sub-linear regret and constraint violation guarantees.
\end{abstract}

\vspace{-0.08in}
\section{Introduction}
\vspace{-0.2em}
Reinforcement Learning (RL) has been successfully applied across a wide range of domains,
including robotics~\cite{tang2025deep}, healthcare~\cite{stasolla2025combined}, large 
language models~\cite{pandit2025certifiable}, and game playing~\cite{dobrovsky2018improving}.
However, many real-world systems are inherently safety-critical and operate in environments
whose state evolution is driven not only by the agent's own actions, but also by 
\emph{exogenous factors outside its control}---competing agents, uncooperative humans, or 
adversarial disturbances. Formally, $s_{h+1} = f(s_h, a_h, \bar{a}_h)+\omega_h$, where $\bar{a}_h$ 
represents the action of an external agent that the protagonist cannot govern. Ignoring 
such factors and treating the environment as fully agent-controlled is not merely a 
modeling simplification---it is a source of systematic safety failures, since the 
worst-case behavior of external actors can violate constraints that appear satisfiable 
under nominal conditions. This motivates the need for a framework that is both 
\emph{safe} and \emph{robust to adversarial exogenous influences}.

Safety in RL is commonly addressed via Constrained Markov Decision Processes 
(CMDPs)~\cite{padakandla2022data}, and robustness via Robust CMDPs (RCMDPs) \cite{kitamura2024near}, which 
require constraints to hold under model mismatches. Most existing RCMDP approaches adopt 
\emph{distributional robustness}, modeling uncertainty through ambiguity sets around a 
nominal model. While effective in certain regimes, these methods rely on strong 
assumptions---access to a nominal model, and the true environment lying within a bounded 
$f$-divergence neighborhood---and primarily address \emph{passive} uncertainty sources 
such as sim-to-real gaps~\cite{ganguly2025efficient}. Crucially, they fail to capture 
\emph{adaptive, policy-dependent} disturbances: the external agent's behavior may shift 
strategically in response to the protagonist's policy, making fixed ambiguity sets 
fundamentally inadequate.

To illustrate, consider an autonomous vehicle merging into a busy lane. Surrounding 
drivers do not inject random noise into the system---they react strategically. An 
aggressive merge may provoke a driver to block; hesitation may prompt others to 
accelerate and deny entry. Even rare but critical reactions, such as a driver refusing 
to yield, can cause safety violations. Such interactions cannot be captured by a 
distributional ambiguity set; they require modeling the external agent as an active 
adversary whose policy co-determines state transitions.

We adopt precisely this perspective. We model exogenous disturbances as arising from an 
explicit \emph{adversarial policy} $\bar{\pi}$ that co-determines state evolution, and 
ask: \emph{how can an agent remain both optimal and safe when external actors behave 
adversarially?} To the best of our knowledge, this is the \emph{first} work to study 
safety-constrained RL under such an adversarial framework, where constraint satisfaction 
must hold against worst-case, policy-dependent perturbations. Formally:
\begin{equation}
    \begin{aligned}
        &\max_{\pi \in \Pi} \min_{\bar{\pi} \in \bar{\Pi}}\; J_{r}(f,\pi,\bar{\pi}),\textbf{s.t.}\quad \min_{\bar{\pi} \in \bar{\Pi}}\;J_{u}(s,\pi,\bar{\pi}) \geq b.
    \end{aligned}
    \label{eqn:opt}
\end{equation}
The protagonist $\pi$ maximizes expected return while ensuring safety under all 
adversarial disturbances; the antagonist $\bar{\pi}$ actively perturbs the environment 
to degrade both. This yields a MARL framework~\cite{zhang2021multi} in which robustness 
and safety are jointly enforced. $J_{g}(f,\pi,\bar{\pi})$ for $g=\{r,u\}$ is formally defined in section \ref{sec:pf}

To solve~\eqref{eqn:opt}, we introduce \texttt{RHC-UCRL} (Robust Hallucinated 
Constrained Upper-Confidence Reinforcement Learning), a model-based algorithm inspired 
by the unconstrained \texttt{RH-UCRL}~\cite{curi2021combining}. Like its predecessor, 
\texttt{RHC-UCRL} separates \emph{epistemic} uncertainty (limited data) from 
\emph{aleatoric} uncertainty (inherent stochasticity), and employs \emph{hallucination}\cite{curi2020efficient}. Informally, hallucination refers to constructing plausible but unobserved transitions that reflect uncertainty in the dynamics, thereby enabling the agent to reason about and guard against adverse outcomes. The magnitude of these hallucinated perturbations is governed by the model's epistemic uncertainty and decreases as the algorithm accumulates more data.

The constrained setting, however, brings challenges absent in the unconstrained case. First, the adversary can now focus its perturbations specifically on inducing constraint violations, rather than merely degrading reward. Second, and more subtly, the worst-case adversarial action for the reward objective and the constraint objective need not coincide---meaning the standard primal-dual approach, which relies on a single adversary solving a scalarized objective, may fail to handle~\eqref{eqn:opt}. To address this, we adopt a \emph{rectified penalty} approach: constraint violations are penalized heavily, while feasible solutions incur no penalty. This surrogate formulation decouples the reward and constraint problems in a principled way, circumventing the breakdown of strong duality.

\textbf{Main Contributions}.
\begin{enumerate}[leftmargin=*]
    \item We propose \texttt{RHC-UCRL}, the first \emph{provably robust} constrained RL 
    algorithm that is (i) computationally efficient, (ii) compatible with deep function 
    approximation, and (iii) deployable in safety-critical settings. A key technical 
    contribution is a rectified penalty formulation that correctly handles the 
    misalignment between reward- and constraint-adversaries, where standard primal-dual 
    methods fail.

    \item We establish that \texttt{RHC-UCRL} achieves sub-linear regret 
    and sub-linear constraint violation---the first such guarantees for 
    constrained RL under adversarial dynamics.
    \item Empirically, \texttt{RHC-UCRL} achieves good reward and maintains feasibility for almost complete duration, unlike RH-UCRL \cite{curi2021combining}.
\end{enumerate}
\vspace{-0.5em}
\subsection{Related Literature}
\label{sec:rel-lit}

\textbf{CMDP}: Primal-dual based approaches with provable performance guarantee have been proposed \cite{ghosh2022provably,ding2023last,ghosh2024towards} to study non-robust CMDP with $O(1/\epsilon^2)$ iteration and sample complexity guarantee. The approaches used the strong duality result and the dynamic programming method which are not possible to apply to the robust CMDP directly

\textbf{Robust Unconstrained MDP}: Robust MDPs have been studied under both known uncertainty sets \cite{iyengar2005robust,wang2021online} and unknown uncertainty sets \cite{shi2024curious,panaganti2022robust,xu2023improved}. Among these works, only \cite{wang2023policy} provides iteration complexity guarantees for robust policy optimization.

A separate line of research has explored adversarial robustness in unconstrained MDPs \cite{curi2021combining,vinitsky2020robust}, where an explicit adversarial agent is introduced to perturb the environment and hinder the learning process of the agent which constitutes a stronger form of robustness.

\textbf{Robust CMDP:} Robust CMDPs (RCMDPs) extend CMDPs to settings with uncertain transition dynamics. Early approaches relied on primal-dual and Lagrangian methods \cite{russel2020robust,mankowitz2020robust,wang2022robust}. However, \cite{wang2022robust} showed that strong duality generally fails in RCMDPs, as the worst-case transition dynamics depend on the policy, rendering standard CMDP techniques ineffective in the robust setting. To address this challenge, \cite{kitamura2024near} proposed an epigraph-based approach for near-optimal policy identification, though it incurs significant computational overhead. Subsequent work improved efficiency by avoiding binary search and establishing stronger iteration guarantees \cite{ganguly2025efficient}. Other studies demonstrated that strong duality can be recovered under restricted policy classes \cite{ganguly2025iteration}. 

Despite these advances, all existing works focus on distributional robustness. In contrast, policy-dependent adversarial settings remain largely unexplored. 
\vspace{-0.05in}
\section{Problem Formulation}
\label{sec:pf}
We consider a stochastic environment with state space $\mathcal{S} \subset \mathbb{R}^{p}$, action space $\mathcal{A}\subset \mathbb{R}^{q}$, adversary action space $\bar{\mathcal{A}} \in \mathbb{R}^{t}$ and the $i.i.d$ additive noise vector $\boldsymbol{\omega}_{h} \in \mathbb{R}^{p}$. We consider $\mathcal{A}$ and $\mathcal{\bar{A}}$ to be compact and state transition dynamics given by:
\vspace{-0.7em}
\begin{align}
     \mathbf{s_{h+1}} = f(\mathbf{s_{h}},\mathbf{a_{h}},\mathbf{\bar{a}_{h}}) +\boldsymbol{\omega_{h}},
    \label{eqn:transition}
 \end{align}
with $f:\mathcal{S}\times \mathcal{A} \times \mathcal{\bar{A}} \to \mathcal{S}$. We assume the true dynamics $f$ is \emph{unknown} and consider the episodic setting over a finite time horizon $H$. After every episode, the system is reset to a fixed state $s_{0}$. In this work the initial state $s_{0}$ is deterministic and fixed. For this work we make the following assumptions (assumption \ref{assume1}) regarding the stochastic environment and unknown dynamics. Note this type of assumption is very standard and has been used in many other works modelling adversaries to be the result of actions taken by an adversarial agent \cite{curi2021combining,curi2020efficient}.

\begin{assumption}
    \label{assume1}
    The function f that determines the dynamics in equation \eqref{eqn:transition} is $L_{f}$-Lipschitz continuous and the noise $\omega_{h}~\forall ~h \in \{0,1,\ldots,H-1\}$ are $i.i.d.$ $\sigma$-Sub-Gaussian.
\end{assumption}

 At each step, the system returns a deterministic reward $r(\mathbf{s}_h, \mathbf{a}_h, \mathbf{\bar{a}}_h)$ and a deterministic utility $u(\mathbf{s}_h, \mathbf{a}_h, \mathbf{\bar{a}}_h)$, where $r:\mathcal{S} \times \mathcal{A} \times \bar{\mathcal{A}} \to [0,1]$ and $u:\mathcal{S} \times \mathcal{A} \times \bar{\mathcal{A}} \to [0,1]$ are known to the agent.

We consider time-homogeneous policies for both agents. The protagonist policy $\pi \in \Pi$ is defined as $\pi:\mathcal{S} \to \mathcal{A}$, selecting actions according to $\mathbf{a}_h = \pi(\mathbf{s}_h)$. Similarly, the adversarial policy $\bar{\pi} \in \bar{\Pi}$ is defined over the same state space as $\bar{\pi}:\mathcal{S} \to \bar{\mathcal{A}}$, selecting actions as $\mathbf{\bar{a}}_h = \bar{\pi}(\mathbf{s}_h)$.

For simplicity, we omit standard extensions such as initial state distributions and time-dependent policies, which can be incorporated using well-established techniques (e.g., \cite{chowdhury2017kernelized}).

The performance of a pair of policies $(\pi,\bar{\pi})$ on a given dynamical system $\tilde{f}$ is the episodic expected sum of returns and expected sum of utility:
\vspace{-0.2em}
\begin{equation}
    \label{eqn:reward_fn}
    \begin{aligned}
        &J_{r}(\tilde{f},\pi,\bar{\pi}) := \mathbb{E}_{\tau_{\tilde{f},\pi,\bar{\pi}}}\left[\sum_{h=0}^{H} r(\mathbf{s}_{h},\mathbf{a}_{h},\mathbf{\bar{a}}_{h})|\mathbf{s}_{0}\right], 
    \end{aligned}
\end{equation}
\begin{equation}
\label{eqn:utility_fn}
    \begin{aligned}
        &J_{u}(\tilde{f},\pi,\bar{\pi}) :=\mathbb{E}_{\tau_{\tilde{f},\pi,\bar{\pi}}}\left[\sum_{h=0}^{H-1} u(\mathbf{s}_{h},\mathbf{a}_{h},\mathbf{\bar{a}}_{h})|\mathbf{s}_{0}\right],
    \end{aligned}
\end{equation}
such that the state transition is given by $\mathbf{s_{h+1}}=\tilde{f}(\mathbf{s}_{h},\mathbf{a}_{h},\bar{\mathbf{a}}_{h}) + \omega_{h}$ and $\tau_{\tilde{f},\pi,\bar{\pi}} = \{(\mathbf{s}_{h-1},\mathbf{a}_{h-1},\mathbf{\bar{a}}_{h-1}),\mathbf{s}_{h}\}_{h=1}^{H}$ is a random trajectory induced by $\boldsymbol{\omega},\tilde{f}$ and $(\pi,\bar{\pi})$.

To incorporate robustness and safety of the system, we need to solve the following problem as given in equation \eqref{eqn:opt}

One way to do is to consider the Lagrangian, and solve the following
\begin{equation}
\label{eqn:obj}
    \begin{aligned}
        \pi^{*} \in \arg \min_{\lambda \geq 0} \max_{\pi \in \Pi} \min_{\bar{\pi} \in \bar{\Pi}} J_{r}(f,\pi,\bar{\pi}) + \lambda(\min_{\bar{\pi} \in \bar{\Pi}}J_{u}(f,\pi,\bar{\pi})-b).
    \end{aligned}
\end{equation}
\vspace{-0.2em}
\textbf{Difficulty of the Primal-Dual Approach}: In the standard CMDP setting, strong 
duality holds under Slater's condition, and the constrained problem reduces to an 
unconstrained one via scalarization: it suffices to optimize the combined reward 
$r + \lambda u$ for an appropriate multiplier $\lambda$. Neither reduction is available 
here. The core obstruction is that the worst-case adversarial policy $\bar{\pi}$ need 
not be the same for the reward objective and the constraint objective. The  two $\min_{\bar{\pi}}$ operators in~\eqref{eqn:obj} cannot be merged into a 
single adversary. Consequently, the Lagrangian cannot be written as a single minimax 
problem, strong duality breaks down, and the scalarized objective $r + \lambda u$ yields 
a \emph{different} adversary than the one in the original problem---making the standard 
primal-dual approach fundamentally inapplicable.

\textbf{Rectified Penalty Framework}: In this work, we solve a different surrogate problem as shown in equation \eqref{eqn:surrogate} and show that the policy pair returned by this algorithm achieves sublinear regret and violation bounds.
\begin{equation}
\label{eqn:surrogate}
    \pi \in \arg \max_{\pi \in \Pi} \min_{\bar{\pi} \in \bar{\Pi}} \left[J_{r}(f,\pi,\bar{\pi}) - \lambda[ b - J_{u}(f,\pi,\bar{\pi})]_{+}\right],
\end{equation}
where $[x]_{+}:=\max(x,0)$ and $\lambda$ is fixed to a high value.

\paragraph{Learning Metric} In this work, the underlying setting is episodic. At each episode $t$, the learning algorithm selects both the agent's policy $\pi_{t}$ and the adversary's policy $\bar{\pi}_{t}$ using the current dynamical model(more details in section \ref{sec:rhc-ucrl}). The pair $(\pi_{t},\bar{\pi}_{t})$ is then played to realize a trajectory $\tau_{f,\pi_{t},\bar{\pi}_{t}}$. The complete algorithm is summarized in Algorithm \ref{algo:rhc_ucrl}. In the \emph{lane-merging} problem, the adversary can be thought of as varied traffic patterns, where in each case the agent learns to find the optimal safe policy.


\paragraph{Statistical Model} We consider a model based RL approach. That is, to learn the dynamics and find a near-optimal robust policy we consider algorithms that model and sequentially learn about $f$ from noisy state observations. We use statistical estimation to probabilistically reason about dynamic models $\tilde{f}$ that are compatible with the observed data $\mathcal{D}_{1:t} = \{ \nu_{f,\pi_{t^{'}},\bar{\pi}_{t^{'}}}\}_{t^{'}=1}^{t}$. This can be done by frequentist estimation of $\mu_{t}(\mathbf{s},\mathbf{a},\mathbf{\bar{a}})$ and confidence $\Sigma_{t}^{2}(\mathbf{s},\mathbf{a},\mathbf{\bar{a}})$ or the bayesian estimation and considering posterior distribution $p(\tilde{f}|\mathcal{D}_{1:t})$ over dynamical models (\cite{Srinivas_2012}). In any case, we need our model to be calibrated. 

\begin{assumption}
\label{assume2}
    The statistical model is calibrated w.r.t f in \eqref{eqn:transition}, so that with $\sigma_{t}(.) = \text{diag}(\Sigma_{t}(.))$ and a non-decreasing sequence of parameters $\{\beta_{t}\}_{t \geq 1} \in \mathbb{R}_{+}$, each depending on $\delta \in (0,1)$, it holds jointly for all $t \geq 1$ and $\mathbf{s,a,\bar{a}} \in \mathcal{S}\times \mathcal{A} \times \mathcal{\bar{A}}$ such that $|f(\mathbf{s,a,\bar{a}}) - \mu_{t-1}(\mathbf{s,a,\bar{a}})| \leq \beta_{t} \sigma_{t-1}(\mathbf{s,a,\bar{a}})$ elementwise with probability at least $1-\delta$
    \label{assump:cali}
\end{assumption}

The calibrated model assumption (Assumption \ref{assume2}) is very common in literature \cite{curi2021combining,Srinivas_2012,chowdhury2017kernelized}. This assumption is crucial for exploration: if the model is not well-calibrated, then leveraging its epistemic uncertainty does not provably guide exploration. For dynamics with bounded norm in a known RKHS, Assumption~2 is satisfied \cite{Srinivas_2012,chowdhury2017kernelized}. In the case of neural network models, one-step-ahead predictions can be recalibrated \cite{malik2019calibratedmodelbaseddeepreinforcement}. 

Finally, we construct the set of plausible models at time $t$ as $\mathcal{M}_{t} =\{\tilde{f} || \tilde{f}(.) - \mu_{t-1}(.)| \leq \beta_{t}\Sigma_{t-1}(.)\}$ . By Assumption 2, we can guarantee that, with high probability, the true dynamics $f \in \mathcal{M}_{t}$, for all time t.
\vspace{-0.2em}
\section{The Robust Constrained H-UCRL Algorithm}
In this section we will discuss our algorithm\footnote{The complete code and related files can be found in \url{https://github.com/Sourav1429/RHC_UCRL.git}} for selecting the policy pair $(\pi_{t},\bar{\pi}_{t})$ at any instant $t$.
\subsection{Optimistic and Pessimistic Policy Evaluation}
Let us denote the optimistic and pessimistic estimate for $J_{g}(f,\pi,\bar{\pi})~\text{s.t.} g\in \{r,u\}$ at any instant $t$ as $J_{g_{t}}^{(o)}(\pi,\bar{\pi})$ and $J_{g_{t}}^{(p)}(\pi,\bar{\pi})$ for $g \in \{r,u\}$ respectively. For constructing these estimates, we make use of the epistemic uncertainty of dynamical model. To find these estimates, we introduce auxiliary function $\eta:\mathcal{S} \times \mathcal{A} \times \mathcal{\bar{A}} \to [-1,1]^{p}$ amd reparameterize the set of plausible models as $\tilde{f} = \mu_{t-1}(.) + \beta_{t}\eta(.)\Sigma_{t-1}(.)$. Using this reparameterization, for $g=r,u$, the optimistic estimate is given by
\vspace{-0.6em}
\begin{equation}
    \begin{aligned}
    \label{eqn:opt_def}
        &J_{g_{t}}^{(o)}(\pi,\bar{\pi}) := \max_{\eta^{(o)}} J_{g}(f^{(o)},\pi,\bar{\pi})~\text{s.t. }g \in\{r,u\},\\
        &\text{ s.t. } f^{(o)}(\mathbf{s},\mathbf{a},\bar{\mathbf{a}}) = \mu_{t-1}(\mathbf{s},\mathbf{a},\mathbf{\bar{a}}) + \\&\beta_{t}\eta^{(o)}(\mathbf{s},\mathbf{a},\mathbf{\bar{a}})\Sigma_{t-1}^{1/2}(\mathbf{s},\mathbf{a},\mathbf{\bar{a}}).
    \end{aligned}
\end{equation}
Similarly , for $g=r,u$, the pessimistic estimate at time $t$  is given by
\begin{equation}
\label{eqn:pes_def}
    \begin{aligned}
        &J_{g_{t}}^{(p)}(\pi,\bar{\pi}) := \min_{\eta^{(p)}} J_{g}(f^{(p)},\pi,\bar{\pi})\\
        &\text{ s.t. } f^{(p)}(\mathbf{s},\mathbf{a},\bar{\mathbf{a}}) = \mu_{t-1}(\mathbf{s},\mathbf{a},\mathbf{\bar{a}}) + \\&\beta_{t}\eta^{(p)}(\mathbf{s},\mathbf{a},\mathbf{\bar{a}})\Sigma_{t-1}^{1/2}(\mathbf{s},\mathbf{a},\mathbf{\bar{a}}).
    \end{aligned}
\end{equation}

For simplicity we denote $J^{(x)}_{g_{t}} = J_{g_{t}}^{(x)}(\pi_{t},\bar{\pi})$ for $x=\{o,p\}$ and $g=\{r,u\}$. Note that the optimistic/pessimistic outcome is selected by the decision variables $\eta^{(o)}/\eta^{(p)}:\mathcal{S}\times \mathcal{A}\times \mathcal{\bar{A}} \to [-1,1]^{p}$. Both equations \eqref{eqn:opt_def} and \eqref{eqn:pes_def} are independent control problems where the decision variables $\eta^{(o)}/\eta^{(p)}$ are hallucinated control policies. 
\vspace{-0.3em}
\subsection{The RHC-UCRL Algorithm}
\label{sec:rhc-ucrl}
Given the definitions of optimistic and pessimistic estimates of $J(f,\pi,\bar{\pi})$, we are now ready to state our algorithm. At each episode $t$, \texttt{RHC-UCRL} selects agent and adversary policies as given by Algorithm \ref{algo:rhc_ucrl} (Line 9 and 10). 

Thus, \texttt{RHC-UCRL} selects the most optimal safe policy as the policy to be played by the agent, whereas the adversary player picks the most pessimistic policy. 
\vspace{-0.4em}
\begin{algorithm}
\caption{RHC-UCRL}
\label{algo:rhc_ucrl}
    \begin{algorithmic}[1]
    \vspace{-0.3em}
        \State \textbf{Input:}  $s_{0}$,  $r:\mathcal{S} \times \mathcal{A} \times \mathcal{\bar{A}} \to [0,1]$, $u:\mathcal{S} \times \mathcal{A} \times \mathcal{\bar{A}} \to [0,1]$, $H$, $\tau$,$\lambda$, $\pi_{0},\bar{\pi}_{0}$
        \For{$t=1,2,\ldots,\tau$}
        \State Reset the system to $s_{0,t+1} = s_{0}$
        \For{$h=1,\ldots,H$}
        \State $s_{(h,t)} = f(s_{(h-1,t)},\pi_{t}(s_{(h-1,t)}),\bar{\pi}_{t}(s_{(h-1,t)})) + \omega_{h,t}$
        \State Collect $\left(s_{h-1},a_{h-1},\bar{a}_{h-1},s_{h},r,u \right)_{t}$
        \EndFor
        \State Update the model dynamics to accurately express $f^{(o)}$ and $f^{(p)}$
        \State $\pi_{t} \in \arg \underset{\pi}{\max}~\underset{\bar{\pi}}{\min}\left(J_{r}^{(o)}-\lambda\left[b-J_{u}^{(o)}\right]_{+}\right)$
        \State $\bar{\pi}_{t} \in \arg \underset{\bar{\pi}}{\min} \left(J^{(p)}(\pi_{t},\bar{\pi})-\lambda\left[b-J_{u}^{(p)}(\pi_{t},\bar{\pi})\right]_{+}\right) $
        \EndFor
    \end{algorithmic}
    \vspace{-0.3em}
\end{algorithm}

\textit{Note that the optimization problems in Algorithm~\ref{algo:rhc_ucrl} (line (9) and (10)) are hard to solve directly. We describe how we solve them efficiently in the experimental Section ~\ref{sec:experiments}.}
\section{Theoretical Results}
In this section, we theoretically analyze the performance of  \texttt{RHC-UCRL} algorithm. First, we use the notion of robust cumulative regret (equation \eqref{eqn:cumu_reg}) and violation (equation \eqref{eqn:cumu_vio})\footnote{Similar notions were used in \cite{kirschner2020distributionally}} for a policy pair $(\pi_{t}, \bar{\pi_{t}})$ which measures the difference in performance from the optimal policy and the amount of violation the policy pairs impose respectively. 
\begin{equation}
    \label{eqn:cumu_reg}
    R_{T}  = \sum_{t=1}^{T} \min_{\bar{\pi} \in \bar{\Pi}} J(f,\pi^{*},\bar{\pi}) - \min_{\bar{\pi} \in \bar{\Pi}} J(f,\pi_{t},\bar{\pi}),
\end{equation}
\vspace{-0.18in}
\begin{equation}
\label{eqn:cumu_vio}
    V_{T} = \sum_{t=1}^{T} (b- J(f,\pi_{t},\bar{\pi}_{t}))_+
\end{equation}

Regret defines the sub-optimality gap over the number of episodes, whereas violation denotes the cumulative cancellation-free constraint violation. Note that our violation bound is stronger than the other violation metric typically considered in the CMDP literature \cite{ghosh2022provably,wei2021provably}, $\sum_{t=1}^{T} (b- J(f,\pi_{t},\bar{\pi}_{t}))$ where policies can violate for $O(T)$ number of episodes with zero violation (e.g., policies alternate between feasibility and infeasibility with the same margin). 

In theorem \ref{thm:1}, we establish that \texttt{RHC-UCRL} achieves sublinear regret, i.e., $R_{T}/T \to 0$ and sublinear violation i.e. $V_{T}/T \to 0$ as $t \to \infty$. Before stating our main theoretical results, we introduce some additional assumptions:
\begin{assumption}
\label{assume3}
    For every episode $t$, the functions $\Sigma_{t}$, the agent`s policy $\pi_{t} \in \Pi$, the adversary`s policy $\bar{\pi}_{t} \in \bar{\Pi}$, the reward function $r(.,.,.)$ and the utility function $u(.,.,.)$ are Lipschitz continuous with respective constants as $L_{\sigma},L_{\pi},L_{\bar{\pi}},L_{r}$ and $L_{u}$.
\end{assumption}
The previous assumption is mild and has been used in non-robust and unconstrained robust model-based RL, see, e.g. (see \cite{curi2020efficient,curi2021combining}). The robust regret, robust violation, and sample complexity rates that we analyze depend on the difficulty of learning the underlying statistical model. Models that are easy to learn typically require fewer samples and allow algorithms to make better decisions sooner. To express the difficulty of learning the imposed calibrated model class, we use the following model-based complexity measure:
\begin{equation}
    \Gamma_{T} = \max_{\tilde{\mathcal{D}}_{1:T}} \sum_{t=1}^{T} \sum_{(\mathbf{s},\mathbf{a},\bar{\mathbf{a}}} \|\Sigma_{t-1}(\mathbf{s},\mathbf{a},\mathbf{\bar{a}})\|_{2}^{2}
\end{equation}
$\Gamma_T$ is known as the information gain 
for the Gaussian process (GP) a kernel-dependent quantity introduced by \cite{srinivas2009gaussian}, that is widely used to characterize the complexity of learning GP models. Sublinear upper bounds on $\Gamma_T$ are known for commonly used kernels, such as  squared-exponential kernels, as well as their compositions (e.g., additive kernels). We use these results to express $\beta_T$ and to upper bound $\Gamma_T$ in Theorem \ref{thm:1}.  The RH-UCRL also expressesed their theoretical guarantees in terms of $\Gamma_T$.

Now we are ready to state the main results of this section\footnote{Lemmas and related proofs in \cite{ganguly2026}}
\begin{theorem}
\label{thm:1}
    Under Assumptions \ref{assume1} to \ref{assume3}, let $C=\left( \left(1+L_{f} + 2L_{\sigma} \right)\sqrt{1 + L_\pi^2 + L_{\bar{\pi}}^2}\right)$, $\lambda = T^{\kappa}$ s.t $\kappa \in (0,1/2)$, $s_{h,t} \in \mathcal{S}$, $a_{h,t} \in \mathcal{A}$ and $\bar{a}_{h,t} \in \mathcal{\bar{A}}$ for all $t,h > 0$. Then for a fixed $H \geq 1$, with probability at least $1-\delta$, the robust cumulative regret of RHC-UCRL is upper bounded by:
    \begin{equation}
        R_{T} = O\left(L_{(r,\lambda,u)}\beta_{T}^{H}C^{H}H^{1.5}\sqrt{T\Gamma_{T}}\right)
    \end{equation}
    and with probability at least $1-\delta$, the robust cumulative violation of RHC-UCRL is upper bounded by
    \begin{equation}
        V_{T} = O\left(L_{u} \beta_{T}^{H}C^{H}(1+\alpha)H^{1.5}\sqrt{T\Gamma_{T}}\right)
    \end{equation}
\end{theorem}
 The Lemmas and proofs related to the theorem~\ref{thm:1} are provided in the Appendix~\ref{sec:thm1}.
Note that this exponential dependency on $C$ also exists in the unconstrained case \cite{curi2021combining}. If the constraints remain absent, then our regret boils down to the regret in the unconstrained case with $L_{r}$ replacing $L_{r,\lambda,u}$. The dependency on $\Gamma_T$ is also of the same order as in the  unconstrained case.

This regret and violation bound shows that $\mathrm{RHC\text{-}UCRL}$ achieves sublinear robust regret and violation when $\beta_T^H \sqrt{\Gamma_T} = o(\sqrt{T})$. \cite{curi2021combining} provides a concrete example of GP models where this condition holds. Also, for linear case, it reduces to $\log(T)$.  The obtained bound also depends on the Lipschitz constants from Assumption~\ref{assume3}, as well as the episode length $H$, which is assumed to be constant. The dependency of the regret and violation bound on the problem dimension is captured in $\Gamma_T$, while $\beta_T$ also depends on $\log(1/\delta)$.

Note that since we do not use the primal-dual method, our proof techniques are fundamentally different compared to the CMDP literature which mostly uses strong duality argument to prove regret and the violation bound \cite{ghosh2022provably,ding2020natural}. 
\begin{figure*}[h!]
    \centering
    \begin{subfigure}[b]{0.45\textwidth}
        \centering
        \includegraphics[width=\textwidth]{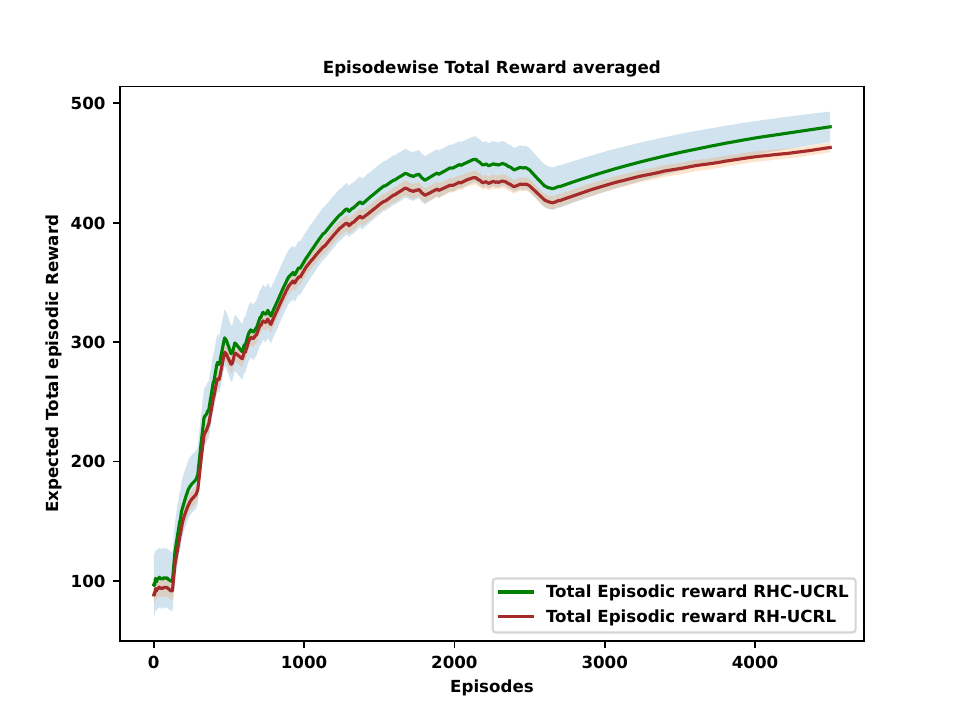}
        \caption{Reward}
        \label{fig:rewardcartpole}
    \end{subfigure}
    \hfill
    \begin{subfigure}[b]{0.45\textwidth}
        \centering
        \includegraphics[width=\textwidth]{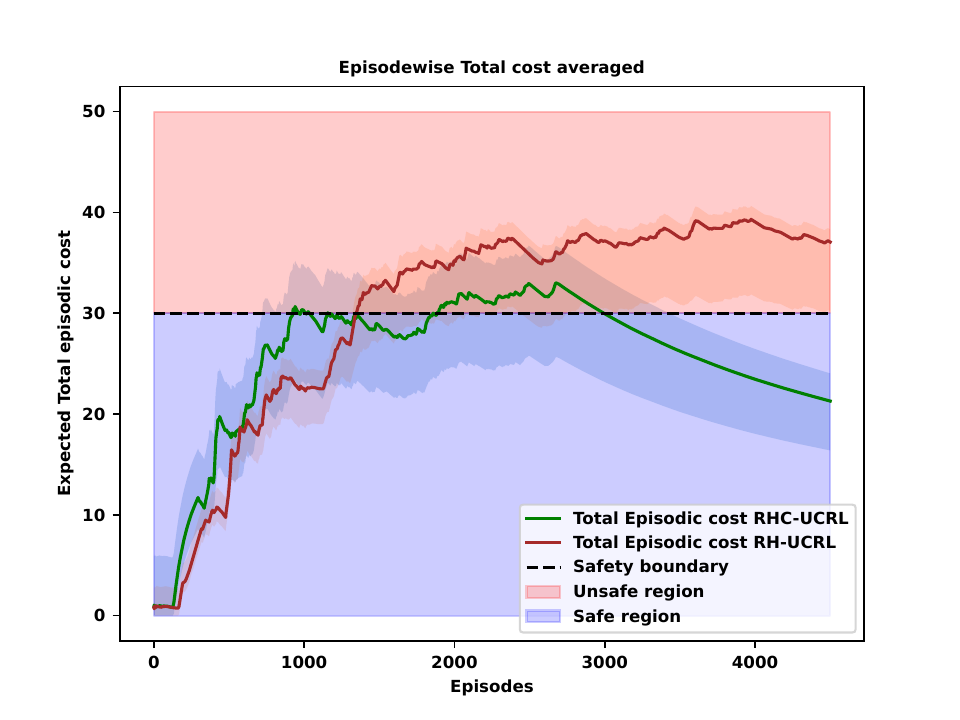}
        \caption{Cost}
        \label{fig:violationcartpole}
    \end{subfigure}
    \caption{Performance of RHC-UCRL and RH-UCRL on the Cartpole-v1 environment.(we use $\lambda=30$)}
    \label{fig:main_results1}
\vspace{-1.3em}
\end{figure*}
\begin{figure*}[h!]
    \centering
    \begin{subfigure}[b]{0.45\textwidth}
        \centering
        \includegraphics[width=\textwidth]{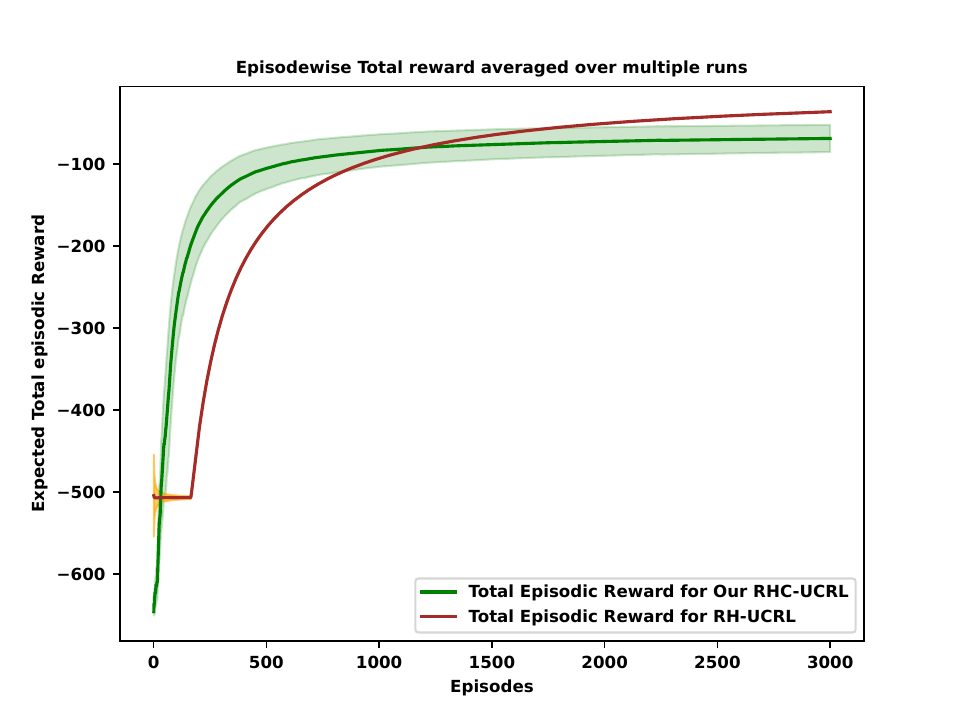}
        \vspace{-0.2in}
        \caption{Reward}
        \label{fig:rewardpendulum}
    \end{subfigure}
    \hfill
    \begin{subfigure}[b]{0.45\textwidth}
        \centering
        \includegraphics[width=\textwidth]{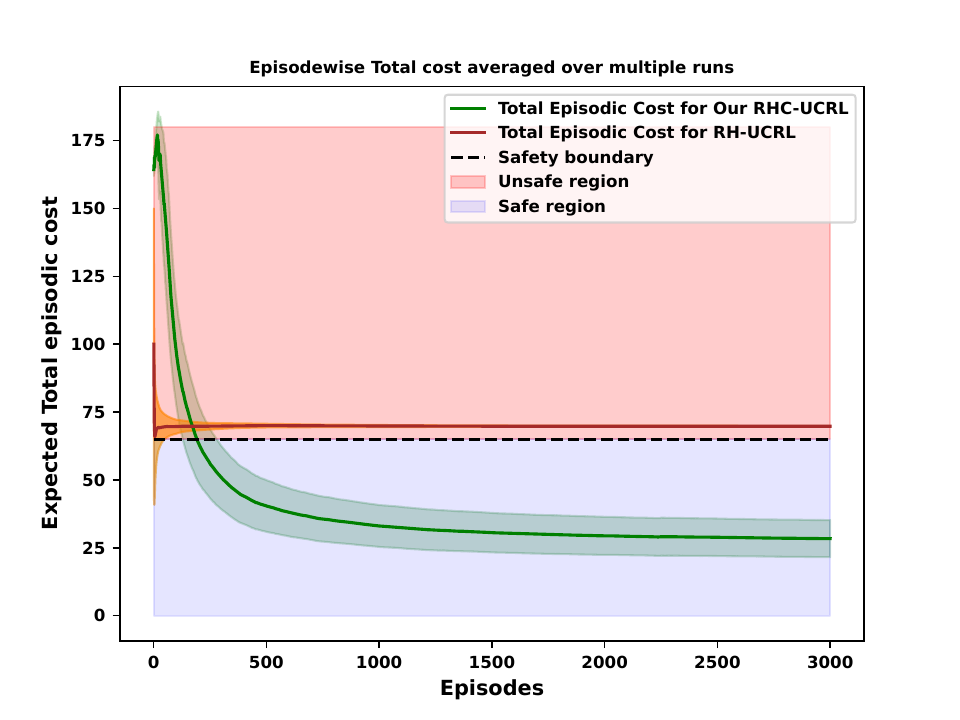}
        \vspace{-0.2in}
        \caption{Cost}
        \label{fig:violationpendulum}
    \end{subfigure}
    \caption{Performance of RHC-UCRL and RH-UCRL on the Pendulum-v1 environment.(we use $\lambda=50$)}
    \label{fig:main_results2}
\vspace{-2em}
\end{figure*}
\vspace{-0.5em}
\section{Experiments}
\label{sec:experiments}
In the following section, we talk about the implementation followed by the empirical results obtained on training our \texttt{RHC-UCRL} algorithm on some standard RL benchamarks. Note during experiments we choose $\lambda$ as a hyperparameter and assume no knowledge about its bounds.
\paragraph{Model Learning} The model is learned (as done by \cite{curi2021combining,curi2020efficient,chua2018deep}) using an ensemble of neural networks, where each member predicts the next state based on the current state and actions of both the agent and adversary. The networks are trained to model state differences using maximum likelihood estimation. The ensemble outputs are combined as a mixture of Gaussians, the average prediction gives the mean next state, while variability across ensemble members captures epistemic uncertainty, and the predicted variances capture aleatoric uncertainty.

\paragraph{Policy Learning} Policies for the agent, adversary, and uncertainty parameters ($\pi,\bar{\pi}$ and $\eta$) are represented using neural networks (as done in \cite{curi2021combining}). The finite-horizon problem is approximated as a discounted infinite-horizon setting, and learning is performed using an actor-critic framework with two critics, an optimistic critic for the agent and a pessimistic critic for the adversary (trained through fitted Q-iteration). Gradients are computed through these critics, with the agent optimizing via gradient ascent and the adversary via gradient descent, enabling a min–max learning procedure. Specifically, we compute the gradients of $\pi$, $\eta^{(o)}$, and $\bar{\pi}_0$ using the learned optimistic critic. The protagonist policy $\pi$ and the optimistic uncertainty parameter $\eta^{(o)}$ are then updated via gradient ascent, while the adversarial policy $\bar{\pi}_0$ is updated via gradient descent. Similarly, using the learned pessimistic critic, we compute the gradients of $\bar{\pi}$ and $\eta^{(p)}$ while keeping $\pi$ fixed. Both the adversarial policy $\bar{\pi}$ and the pessimistic uncertainty parameter $\eta^{(p)}$ are then updated via gradient descent.
\vspace{-0.2em}
\paragraph{Empirical Results} We evaluate the proposed RHC-UCRL algorithm on two benchmark environments, Pendulum-v1 and CartPole-v1, under adversarial perturbations, where at each step an adversary selects perturbations that interact with the agent’s action, capturing the policy–adversary interaction in our formulation. Within this adversarial setting, to assess the performance of RHC-UCRL, we compare it against the baseline: 
RH-UCRL~\cite{curi2021combining}, an unconstrained variant that focuses on reward maximization. Note that we have not compared with \cite{ganguly2025efficient,kitamura2024near,ma2025rectified} as they are for distributionally RCMDP.

\textbf{CartPole:} The CartPole-v1 environment consists of balancing a pole on a moving cart, where the agent selects between two discrete actions. The state includes cart position, cart velocity, pole angle, and pole angular velocity. The reward is obtained by keeping the pole balanced. The adversary perturbs the cart velocity and pole angular velocity before moving the next state. The cost is defined as the absolute distance from the center position. 

\textbf{Pendulum:} The Pendulum-v1 environment is a continuous control task where the agent applies torque to keep the pendulum upright. The state consists of orientation and angular velocity, and the action corresponds to the applied torque. The reward depends on the deviation from the upright position. The adversary perturbs the angle and angular velocity before moving the next state. The cost is defined based on pendulum height: a penalty of 1 is incurred when the height drops below 0.7, and 0 otherwise. 

In both environments, the goal is to maximize cumulative reward subject to a constraint on the cumulative cost.

\textbf{Results:} 
For CartPole-v1, Fig.~\ref{fig:rewardcartpole} shows that RHC-UCRL approaches a reward close to 500 around 4000 episodes, following a steady increasing trend toward the maximum reward. RH-UCRL achieves a reward performance close to RHC-UCRL, following a similar increasing trend, but remains consistently slightly lower throughout training. Fig.~\ref{fig:violationcartpole} shows the cost with safety threshold 30. RHC-UCRL remains below the threshold, while RH-UCRL violates the constraint persistently after approximately 1000 episodes and does not return to the safe region thereafter.

For Pendulum-v1, Fig.~\ref{fig:rewardpendulum} shows that RHC-UCRL performance stabilizes around 1000 episodes. RH-UCRL achieves higher reward than RHC-UCRL but remains in the unsafe domain throughout the evaluation stage, as shown in Fig.~\ref{fig:violationpendulum}. In contrast, RHC-UCRL consistently maintains the cost within the safe region. Overall, RHC-UCRL achieves strong reward performance while ensuring constraint satisfaction across both environments.

\vspace{-0.5em}
\section{Conclusions and Future work}
 In this work, we propose a \texttt{RHC-UCRL} algorithm that considers changes in the environment as an adversary that willfully tries to cripple the learning of the agent. Our theoretical results show that \texttt{RHC-UCRL} achieves a policy with sublinear regret and violation guarantees. Empirically, we demonstrate that \texttt{RHC-UCRL} achieves higher reward while consistently satisfying safety constraints.

 Implementing the algorithm for real-life applications and more simulations in other environments constitutes an important future research direction. Reducing the regret guarantee has been left for the future. Reducing the violation bound to zero while keeping the sub-linear regret also constitutes an important future research direction.

\bibliographystyle{IEEEtran}
\bibliography{biblography}
\onecolumn
\appendix
\sloppy
\section{Proofs}
\begin{lemma}
\label{lem:transition_bind}
    (Adapted from Corollary 1 in \cite{efficient_model_based_rl_2020}) Based on assumptions \ref{assume1} and \ref{assume3}, for every $\mathbf{s,s^{'}} \in \mathcal{S}$ the following ineqlity holds
    \begin{equation}
        \begin{aligned}
            \|f(\mathbf{s},\pi(\mathbf{s}),\bar{\pi}(\mathbf{s})) - f(\mathbf{s}^{'},\pi(\mathbf{s}^{'}),\bar{\pi}(\mathbf{s}^{'})) \| \leq L_{f}\sqrt{1+L_{\pi}^2+L_{\bar{\pi}}^2}\|\mathbf{s}-\mathbf{s}^{'}\|_{2}.
        \end{aligned}
    \end{equation}
\end{lemma}

\begin{proof}
\begin{align}
&\|f(\mathbf{s},\pi(\mathbf{s}),\bar{\pi}(\mathbf{s}))
- f(\mathbf{s}',\pi(\mathbf{s}'),\bar{\pi}(\mathbf{s}'))\| \label{eqn:9}\\
&\leq L_f \sqrt{
\|\mathbf{s}-\mathbf{s}'\|_2^2
+\|\pi(\mathbf{s})-\pi(\mathbf{s}')\|_2^2
+\|\bar{\pi}(\mathbf{s})-\bar{\pi}(\mathbf{s}')\|_2^2
} \label{eqn:10}\\
&\leq L_f \sqrt{
\|\mathbf{s}-\mathbf{s}'\|_2^2
+L_\pi^2\|\mathbf{s}-\mathbf{s}'\|_2^2
+L_{\bar{\pi}}^2\|\mathbf{s}-\mathbf{s}'\|_2^2
} \label{eqn:11}\\
&= L_f\sqrt{1+L_\pi^2+L_{\bar{\pi}}^2}\,\|\mathbf{s}-\mathbf{s}'\|_2 . \label{eqn:12}
\end{align}

    The first inequality(eqn. \eqref{eqn:10}) holds due to the Lipschitz continuity of f. The second inequality (eqn. \eqref{eqn:11}) is due to Lipschitz continuity of $\pi$ and $\bar{\pi}$, which we assume in Assumptions \ref{assume1} and \ref{assume3}. This completes the proof.
\end{proof}

\begin{lemma}
\label{lem:reward_bind}
    (Adapted from Lemma 3 in \cite{efficient_model_based_rl_2020}). Assuming the assumptions \ref{assume1} and \ref{assume3} are true, the following inequality holds:
    \begin{equation}
        \begin{aligned}
            |J_{r}(f,\pi,\bar{\pi}) - J_{r}(\tilde{f},\pi,\bar{\pi})| \leq L_{r}\sqrt{1+L_{\pi}^{2}+L_{\bar{\pi}}^2}\sum_{h=0}^{H}\mathbb{E}\left[\|\mathbf{s}_{h}-\tilde{\mathbf{s}}_{h}\|\right]
        \end{aligned}
    \end{equation}
\end{lemma}
\begin{proof}
\begin{align}
        |J_{r}(f,\pi,\bar{\pi}) - J_{r}(\tilde{f},\pi,\bar{\pi})| = \left|\mathbb{E}\left[\sum_{h=0}^{H}r(\mathbf{s,a,\bar{a}}) - \sum_{h=0}^{H}r(\tilde{\mathbf{s}},\tilde{\mathbf{a}},\tilde{\mathbf{\bar{a}}})\right]\right|\label{eq:14}\\
        =\left|\mathbb{E}\left[\sum_{h=0}^{H}(r(\mathbf{s,a,\bar{a}}) - r(\tilde{\mathbf{s}},\tilde{\mathbf{a}},\tilde{\mathbf{\bar{a}}})\right]\right| \label{eq:15}\\
        =\left|\sum_{h=0}^{H}\mathbb{E}\left[(r(\mathbf{s,a,\bar{a}}) - r(\tilde{\mathbf{s}},\tilde{\mathbf{a}},\tilde{\mathbf{\bar{a}}})\right]\right| \label{eq:16}\\
        \leq L_{r}\sqrt{1+L_{\pi}^2+L_{\bar{\pi}}^2}\sum_{h=0}^{H}\mathbb{E}\left[\|\mathbf{s}_{h} - \tilde{\mathbf{s}}_{h}\|_{2}\right]. \label{eq:17}
\end{align}
The first equality (eq.~\eqref{eq:14}) is through the definition of $J(f,\pi,\bar{\pi})$, the second equality (eq.~\eqref{eq:15}) is by converting the difference of total sum of utility following $f$ and $\tilde{f}$. The third equality (eq.~\eqref{eq:16}) is obtained by converting the difference of the total utility into element wise difference as both runs for $H-1$ steps. The final inequality (eq.~\eqref{eq:17})  is due to Lipschitz property of utility, $\pi$ and $\bar{\pi}$. This completes the proof.
\end{proof}

\begin{lemma}
\label{lem:utility_bind}
    Assuming that assumptions \ref{assume1} and \ref{assume3} holds, the following inequality is true:
    \begin{equation}
        \begin{aligned}
            |J_{u}(f,\pi,\bar{\pi}) - J_{u}(\tilde{f},\pi,\bar{\pi})| \leq L_{u}\sqrt{1+L_{\pi}^{2}+L_{\bar{\pi}}^2}\sum_{h=0}^{H}\mathbb{E}\left[\|\mathbf{s}_{h}-\tilde{\mathbf{s}}_{h}\|\right]
        \end{aligned}
    \end{equation}
\end{lemma}
\begin{proof}
    \begin{align}
        |J_{u}(f,\pi,\bar{\pi}) - J_{u}(\tilde{f},\pi,\bar{\pi})| = \left|\mathbb{E}\left[\sum_{h=0}^{H}u(\mathbf{s,a,\bar{a}}) - \sum_{h=0}^{H}u(\tilde{\mathbf{s}},\tilde{\mathbf{a}},\tilde{\mathbf{\bar{a}}})\right]\right|\label{eq:19}\\
        =\left|\mathbb{E}\left[\sum_{h=0}^{H}(u(\mathbf{s,a,\bar{a}}) - u(\tilde{\mathbf{s}},\tilde{\mathbf{a}},\tilde{\mathbf{\bar{a}}})\right]\right|\label{eq:20}\\
        =\left|\sum_{h=0}^{H}\mathbb{E}\left[(u(\mathbf{s,a,\bar{a}}) - u(\tilde{\mathbf{s}},\tilde{\mathbf{a}},\tilde{\mathbf{\bar{a}}})\right]\right|\label{eq:21}\\
        \leq L_{u}\sqrt{1+L_{\pi}^2+L_{\bar{\pi}}^2}\sum_{h=0}^{H}\mathbb{E}\left[\|\mathbf{s}_{h} - \tilde{\mathbf{s}}_{h}\|_{2}\right]\label{eq:22}
    \end{align}
The first equality (eq.~\eqref{eq:19}) is through the definition of $J_{u}(f,\pi,\bar{\pi})$, the second equality (eq.~\eqref{eq:20}) is by converting the difference of total sum of rewards following $f$ and $\tilde{f}$. The third equality (eq.~\eqref{eq:21}) is obtained by converting the difference of the total rewards into element wise difference as both runs for $H-1$ steps. The final inequality (eq.~\eqref{eq:22}) is due to Lipschitz property of reward, $\pi$ and $\bar{\pi}$. This completes the proof
\end{proof}

\begin{lemma}
    \label{lem:robust_utility_bind}
    Assuming that assumptions \ref{assume1}, \ref{assume2} and \ref{assume3} are true, then from Lemma \ref{lem:utility_bind} following result holds
    \begin{equation}
        \begin{aligned}
            |\min_{\bar{\pi} \in \bar{\Pi}} J_{u}(f,\pi,\bar{\pi}) - \min_{\bar{\pi} \in \bar{\Pi}} J_{u}(\tilde{f},\pi,\bar{\pi})| \leq L_{u}\sqrt{1+L_{\pi}^2+L_{\bar{\pi}}^2}\sum_{h=0}^{H}\mathbb{E}\left[\|\mathbf{s}_{h} - \tilde{\mathbf{s}}_{h}\|_{2}\right]
        \end{aligned}
    \end{equation}
\end{lemma}
\begin{proof}
        \begin{align}
            |\min_{\bar{\pi} \in \bar{\Pi}} J_{u}(f,\pi,\bar{\pi}) - \min_{\bar{\pi} \in \bar{\Pi}} J_{u}(\tilde{f},\pi,\bar{\pi})| \leq \sup_{\bar{\pi} \in \bar{\Pi}} |J_{u}(f,\pi,\bar{\pi}) - J_{u}(\tilde{f},\pi,\bar{\pi})|\label{eq:23}\\
            \leq L_{u}\sqrt{1+L_{\pi}^2+L_{\bar{\pi}}^2}\sum_{h=0}^{H}\mathbb{E}\left[\|\mathbf{s}_{h} - \tilde{\mathbf{s}}_{h}\|_{2}\right]\label{eq:24}
        \end{align}
    The first inequality (eq.\eqref{eq:23}) is true since the difference between any two functions is always less than the supremum of the difference between the absolute values of the functions \cite{shalev2014understanding}. The second inequality (eq.\eqref{eq:24}) follows from Lemma \ref{lem:utility_bind} and hence, the proof follows.
\end{proof}

\begin{lemma}
\label{lem:state_bind}
    Under assumptions \ref{assume1}, \ref{assume2} and \ref{assume3}, for all episodes $t \geq 1$, any $\eta \in [-1,1], h \in \{0,\ldots,H-1\},\pi \in \Pi~\text{and } \bar{\pi} \in \bar{\Pi}$ the following inequality holds:
    \begin{equation}
    \begin{aligned}
        \|\textbf{s}_{h,t} - \tilde{\textbf{s}}_{h,t}\| \leq 2\beta_{t-1} \left(\left(1+L_{f} + 2\beta_{t-1} L_{\sigma} \right)\sqrt{1 + L_\pi^2 + L_{\bar{\pi}}^2}\right)^{h-1}\\ 
        \sum_{h'=0}^{h-1}\left\|\sigma_{t-1}\big(s_{h',t},\pi_t(s_{h',t}),\bar{\pi}_t(s_{h',t})\big)\right\|_2.
     \end{aligned}
    \end{equation}
\end{lemma}
\begin{proof}
    To avoid notational complexity, we use $L_{f,\pi} = L_{f}\sqrt{1+L_{\pi}^{2}+L_{\bar{\pi}}^2}$ and $L_{\sigma,\pi} =L_{\sigma}\sqrt{1+L_{\pi}^{2}+L_{\bar{\pi}}^2} $. 

    We prove by induction that
    \begin{equation}
    \label{eqn:ind_hyp}
        \|s_{h,t} - \tilde{s}_{h,t}\|_{2} \leq 2.\beta_{t-1} \sum_{h^{'}=0}^{h-1} (L_{f,\pi} + 2.\beta_{t-1}L_{\sigma,\pi})^{h-1-h^{'}}.\|\sigma_{t-1}^{\pi_{t},\bar{\pi}_{t}}(\mathbf{s}_{h^{'},t})\|
    \end{equation}
    
    For $h=0$, clearly $s_{0,t} = \tilde{s_{0,t}}$. This is because we treat our initial state as fixed.
    We know, $s_{h+1,t} = f(s_{h,t},\pi_{t},\bar{\pi}_{t})$. Thus, we can expand the LHS for $h+1$ as:
    \hspace*{-1.8em}
        \begin{align}
            \|\mathbf{s}_{h+1,t} - \tilde{\mathbf{s}}_{h+1,t}\| = \|f(\mathbf{s}_{h,t},\pi_{t},\bar{\pi}_{t}) - \tilde{f}(\tilde{\mathbf{s}}_{h,t},\pi_{t},\bar{\pi}_{t})\|\label{eq:27}\\
            = \|f(\mathbf{s}_{h,t},\pi_{t},\bar{\pi}_{t}) - \tilde{f}(\tilde{\mathbf{s}}_{h,t},\pi_{t},\bar{\pi}_{t})+f(\tilde{\mathbf{s}}_{h,t},\pi_{t},\bar{\pi}_{t}) - f(\tilde{\mathbf{s}}_{h,t},\pi_{t},\bar{\pi}_{t})\|\label{eq:28}\\
            \leq \|f(\mathbf{s}_{h,t},\pi_{t},\bar{\pi}_{t}) - f(\tilde{\mathbf{s}}_{h,t},\pi_{t},\bar{\pi}_{t}) \|+\|f(\tilde{\mathbf{s}}_{h,t},\pi_{t},\bar{\pi}_{t}) - \tilde{f}(\tilde{\mathbf{s}}_{h,t},\pi_{t},\bar{\pi}_{t}) \|\label{eq:29}\\
            \leq L_{f}\sqrt{1+L_{\pi}^{2}+L_{\bar{\pi}}^2}\|\mathbf{\mathbf{s}}-\mathbf{\mathbf{\tilde{s}}}\|_{2} + 2 \beta_{t-1}\|\sigma_{t-1}^{\pi_{t},\bar{\pi}_{t}}(\tilde{\mathbf{s}}_{h,t})\|_{2}\label{eq:30}\\
            =L_{f,\pi}\|\mathbf{s}_{h,t} - \tilde{\mathbf{s}}_{h,t}\|_{2} + 2\beta_{t-1}\|\sigma_{t-1}^{\pi_{t},\bar{\pi}_{t}}(\tilde{\mathbf{s}}_{h,t})+\sigma_{t-1}^{\pi_{t},\bar{\pi}_{t}}(\mathbf{s}_{h,t})-\sigma_{t-1}^{\pi_{t},\bar{\pi}_{t}}(\mathbf{s}_{h,t})\|_{2}\label{eq:31}\\
            \leq L_{f,\pi}\|\mathbf{s}_{h,t} - \tilde{\mathbf{s}}_{h,t}\|_{2} + 2\beta_{t-1}(\|\sigma_{t-1}^{\pi_{t},\bar{\pi}_{t}}(\tilde{\mathbf{s}}_{h,t})-\sigma_{t-1}^{\pi_{t},\bar{\pi}_{t}}(\mathbf{s}_{h,t})\|_{2} + \|\sigma_{t-1}^{\pi_{t},\bar{\pi}_{t}}(\mathbf{s}_{h,t})\|_{2})\label{eq:32}\\
            \leq  L_{f,\pi}\|\mathbf{s}_{h,t} - \tilde{\mathbf{s}}_{h,t}\|_{2} + 2\beta_{t-1}L_{\sigma}\sqrt{1+L_{\pi}^{2}+L_{\bar{\pi}}^2}\|\mathbf{s}_{h,t} - \tilde{\mathbf{s}}_{h,t}\|_{2}+ 2\beta_{t-1}\|\sigma_{t-1}^{\pi_{t},\bar{\pi}_{t}}(\mathbf{s}_{h,t})\|_{2}\label{eq:33}\\
            =(L_{f,\pi}+2\beta_{t-1}L_{\sigma,\pi})\|\mathbf{s}_{h,t} - \tilde{\mathbf{s}}_{h,t}\|_{2}+ 2\beta_{t-1}\|\sigma_{t-1}^{\pi_{t},\bar{\pi}_{t}}(\mathbf{s}_{h,t})\|_{2}\label{eq:34}\\
            \leq 2\beta_{t-1}\sum_{h'=0}^{(h+1)-1} \left(\left(L_f + 2\beta_{t-1} L_\sigma \right)\sqrt{1 + L_\pi^2 + L_{\bar{\pi}}^2}\right)^{(h+1)-1-h^{'}}\left\|\sigma_{t-1}^{\pi_{t},\bar{\pi}_{t}}(\mathbf{s}_{h,t})\right\|_2\label{eq:35}
        \end{align} 
    Finally, notice that $(L_{f,\pi} +2\beta_{t-1}L_{\sigma,\pi})^{(h-1-h^{'})} < (1+L_{f,\pi}+2\beta_{t-1}L_{\sigma,\pi})^{h-1-h^{'}} \leq (1+L_{f,\pi}+2\beta_{t-1}L_{\sigma,\pi})^{h-1}$ and the result follows when the above inequality is combined with final inequality from equation \eqref{eq:35}.

    The first inequality (eq.~\eqref{eq:29}) is due to triangular inequality, the second inequality (eq.~\eqref{eq:30}) is a direct application of Lemma \ref{lem:transition_bind} and the second term comes from the assumption that both $f(\tilde{\mathbf{s}}_{h,t},\pi_{t},\bar{\pi}_{t}) \in \mathcal{M}$ and $\tilde{f}(\tilde{\mathbf{s}}_{h,t},\pi_{t},\bar{\pi}_{t}) \in \mathcal{M}$. So $\|f(\tilde{\mathbf{s}}_{h,t},\pi_{t},\bar{\pi}_{t}) - \tilde{f}(\tilde{\mathbf{s}}_{h,t},\pi_{t},\bar{\pi}_{t}) \| = \|f(\tilde{\mathbf{s}}_{h,t},\pi_{t},\bar{\pi}_{t})- \mu(\tilde{\mathbf{s}}_{h,t},\pi_{t},\bar{\pi}_{t}) +\mu(\tilde{\mathbf{s}}_{h,t},\pi_{t},\bar{\pi}_{t}) - \tilde{f}(\tilde{\mathbf{s}}_{h,t},\pi_{t},\bar{\pi}_{t})\|$. Now take the triangular inequality to get $\|f(\tilde{\mathbf{s}}_{h,t},\pi_{t},\bar{\pi}_{t}) - \tilde{f}(\tilde{\mathbf{s}}_{h,t},\pi_{t},\bar{\pi}_{t}) \| \leq \|f(\tilde{\mathbf{s}}_{h,t},\pi_{t},\bar{\pi}_{t})- \mu(\tilde{\mathbf{s}}_{h,t},\pi_{t},\bar{\pi}_{t})\| + \|\mu(\tilde{\mathbf{s}}_{h,t},\pi_{t},\bar{\pi}_{t}) - \tilde{f}(\tilde{\mathbf{s}}_{h,t},\pi_{t},\bar{\pi}_{t})\| = 2\beta_{t-1}\sigma_{t-1}(\tilde{\mathbf{s}}_{h,t},\pi_{t},\bar{\pi}_{t})$.

    The third inequality (eq.~\eqref{eq:32}) is by triangular inequality of $\|.\|_{2}$. The fourth inequality (eq.~\eqref{eq:33}) comes from Lipschitz property of $\sigma(.)$, $\pi$ and $\bar{\pi}$. The final inequality (eq.~\eqref{eq:35}) is obtained by replacing inductive hypothesis with equation \eqref{eqn:ind_hyp}
\end{proof}

\begin{lemma}
\label{lem:opt_viol_bound}
 Assuming $\pi_{t}$ is the chosen policy at instant $t$ by Algorithm \ref{algo:rhc_ucrl} and $c > 2$ , the following inequality holds
    \begin{equation}
        [b - \min_{\bar{\pi}}J_{u}^{(o)}(\pi_{t},\bar{\pi})]_{+} \leq cR_{\max}/\lambda
    \end{equation}
\end{lemma}
\begin{proof}
    \begin{align}
         &\lambda[b - \min_{\bar{\pi}}J_{u}^{(o)}(\pi_{t},\bar{\pi})]_{+} \leq \min_{\bar{\pi}}J_{r}^{(o)}(\pi_{t},\bar{\pi})+\lambda[b-J_{u}^{(o)}(\pi_{t},\bar{\pi})]_{+}\label{eq:38} \\
         &-\min_{\bar{\pi}}J_{r}^{(o)}(\pi^{*},\bar{\pi}) - \lambda[b-J_{u}^{(o)}(\pi^{*},\bar{\pi})]_{+} +cR_{\max}\nonumber\\
        &\leq cR_{\max}\label{eq:39} 
    \end{align}
    The first inequality (eq.\eqref{eq:38}) is true because $\pi^{*}$, being the optimal policy, returns the highest value and the highest value cannot be more than $R_{max}$. 
    Thus, $[b - \min_{\bar{\pi}}J_{u}^{(o)}(\pi_{t},\bar{\pi})]_{+} \leq cR_{\max}/\lambda$
\end{proof}

\begin{lemma}
\label{lem:robust_viol_bound}
    Assuming assumptions \ref{assume2} holds
    \[
    \alpha = \frac{cR_{\max}}{\lambda.2L_{u}H \beta_{T}^{H}C^{H}\sum_{h^{'}=0}^{H-1}\mathbb{E}\left[ \|\sigma_{t-1}\big(s_{h',t},\pi_t(s_{h',t}),\bar{\pi}_t(s_{h',t})\big)\|_{2}\right]},\]
    $\pi^{*}$ be the optimal policy solving equation \eqref{eqn:opt} and $\pi_{t}$ and $\bar{\pi_{t}}$ be the policies selected by Algorithm \ref{algo:rhc_ucrl} at instant $t$, then the following inequality holds
    \begin{equation}
        \begin{aligned}
            [b-J_{u}(f,\pi_{t},\bar{\pi}_{t})]_{+} \leq 2L_{u}H \beta_{T}^{H}C^{H}(1+\alpha) \sum_{h^{'}=0}^{H}\mathbb{E}\left[ \|\sigma_{t-1}\big(s_{h',t},\pi_t(s_{h',t}),\bar{\pi}_t(s_{h',t})\big)\|_{2}\right].
        \end{aligned}
    \end{equation}
\end{lemma}
\begin{proof}
\begin{align}
[b-J_{u}(f,\pi_{t},\bar{\pi}_{t})]_{+}
=[b-\min_{\bar{\pi}}J_{u}^{(o)}(\pi_{t},\bar{\pi})
+\min_{\bar{\pi}}J_{u}^{(o)}(\pi_{t},\bar{\pi})
-J_{u}(f,\pi_{t},\bar{\pi}_{t})]_{+} \label{eq:41}\\
\leq [b-\min_{\bar{\pi}}J_{u}^{(o)}(\pi_{t},\bar{\pi})]_{+}
+ |\min_{\bar{\pi}}J_{u}^{(o)}(\pi_{t},\bar{\pi}_{t})
- J_{u}(f,\pi_{t},\bar{\pi}_{t})| \label{eq:42}\\
\leq [b-\min_{\bar{\pi}}J_{u}^{(o)}(\pi_{t},\bar{\pi})]_{+} + |J_{u}^{(o)}(\pi_{t},\bar{\pi}_{t})
- J_{u}(f,\pi_{t},\bar{\pi}_{t})|\label{eq:43} \\
\leq \frac{cR_{\max}}{\lambda}
+ L_{u}\sqrt{1+L_{\pi}^{2}+L_{\bar{\pi}}^{2}}
\sum_{h=0}^{H}\mathbb{E}\!\left[\|s_{h}-\tilde{s}_{h}\|_{2}\right] \label{eq:44}\\
\leq \frac{cR_{\max}}{\lambda}
+ L_{u,\pi}
\sum_{h=0}^{H}\mathbb{E}\!\Bigg[2\beta_{t-1}
\Big(\left(1+L_{f}+2\beta_{t-1}L_{\sigma}\right) \nonumber\\
\cdot \sqrt{1+L_{\pi}^{2}+L_{\bar{\pi}}^{2}}\Big)^{h-1}
\sum_{h'=0}^{h-1}
\left\|\sigma_{t-1}\big(s_{h',t},
\pi_t(s_{h',t}),
\bar{\pi}_t(s_{h',t})\big)\right\|_{2}
\Bigg]\label{eq:45}\\
\leq \frac{cR_{\max}}{\lambda}
+ 2L_{u}\beta_{T}^{H}C^{H}
\sum_{h=0}^{H}\mathbb{E}\!\left[
\sum_{h'=0}^{h}
\left\|\sigma_{t-1}\big(s_{h',t},
\pi_t(s_{h',t}),
\bar{\pi}_t(s_{h',t})\big)\right\|_{2}
\right] \label{eq:46}\\
\leq \frac{cR_{\max}}{\lambda}
+ 2L_{u}H\beta_{T}^{H}C^{H}
\sum_{h'=0}^{H}\mathbb{E}\!\left[
\left\|\sigma_{t-1}\big(s_{h',t},
\pi_t(s_{h',t}),
\bar{\pi}_t(s_{h',t})\big)\right\|_{2}
\right]\label{eq:47}\\
= 2L_{u}H\beta_{T}^{H}C^{H}
\sum_{h'=0}^{H}\mathbb{E}\!\left[
\left\|\sigma_{t-1}\big(s_{h',t},
\pi_t(s_{h',t}),
\bar{\pi}_t(s_{h',t})\big)\right\|_{2}
\right]\nonumber\\ \left(
\frac{cR_{\max}}
{\lambda \cdot 2L_{u}H\beta_{T}^{H}C^{H}
\sum_{h'=0}^{H}\mathbb{E}\!\left[
\left\|\sigma_{t-1}\big(s_{h',t},
\pi_t(s_{h',t}),
\bar{\pi}_t(s_{h',t})\big)\right\|_{2}
\right]}
+ 1\right)\label{eq:48} \\
= 2L_{u}H\beta_{T}^{H}C^{H}(1+\alpha)
\sum_{h'=0}^{H}\mathbb{E}\!\left[
\left\|\sigma_{t-1}\big(s_{h',t},
\pi_t(s_{h',t}),
\bar{\pi}_t(s_{h',t})\big)\right\|_{2}
\right].\label{eq:49}
 \end{align}
    The first equality (eq.~\eqref{eq:41}) comes from adding and subtracting $\min_{\bar{\pi}}J_{u}^{(o)}(\pi_{t},\bar{\pi}_{t})$. The first inequality (eq.~\eqref{eq:42}) is due to the triangle inequality of $||.||_{\infty} = \max{a_{1},a_{2}}$, here $a_{2}=0$ and then invoking $[a]_{+}=\max(a,0)<|a|$. This inequality is true because $a \leq |a|~\text{and}~0 \leq |a|$. The second inequality(eq.~\eqref{eq:43})  is due to replacement of $\min_{\bar{\pi}} J_{u}^{(o)}(\pi_{t},\bar{\pi})$ with any adversarial policy $\bar{\pi}$. This is true since if we place the variable which returns the minimum functional value with any other variable in the candidate set, while the other function is fixed, the difference between the two values is bound to increase. Now, the third inequality (eq.~\eqref{eq:44}) is a direct application of Lemma \ref{lem:opt_viol_bound} and Lemma \ref{lem:utility_bind}. The fourth inequality (eq.~\eqref{eq:45})  is a direct application of Lemma \ref{lem:state_bind}. The fifth inequality (eq.~\eqref{eq:46}) is by assuming $C^{H}$ which is the highest power of $C$ as $C^{h}$ was increasing term by term as $C^0$ when $h=1$, $C^{1}$ when $h=2$, thus, we replace all these $C^{H}$ with the highest value of $h$ which is $H$ and take that common outside the summation. The sixth inequality (eq.~\eqref{eq:47}) is obtained by taking the outer summation $\sum_{h=0}^{H}$ inside expectation as expectation is a linear operator and then perform merging of the two summations into one by causing change of variables which spit out an additional $(H-h+1)$ term inside the expectation. Now we consider this term $(H-h+1)$ changes as $h$ increases, so we replace each with the maximum value $H$ and take that common. Next equality (eq.~\eqref{eq:49}) step is simply by taking the relevant terms common such that we can replace the second term with $(1+\alpha)$ which follows from the definition of $\alpha$ above at the start of the lemma. 
\end{proof}

\begin{lemma}
    \label{lem:opt_pes_rew_util}
    Let $\pi^{*}$ be the optimal feasible policy and the solution for equation \ref{eqn:opt} and let $\pi_{t}$ and $\bar{\pi}_{t}$ be the polices selected by $RHC-UCRL$ at time $t$. Then under assumption \ref{assume2} and $L_{(r,\lambda,u)} = (L_{r} +\lambda L_{u})$, the following inequality (equation \eqref{opt_pes_rew}) hold with probability at least $1-\delta_{r}$.
    \begin{equation}
    \label{opt_pes_rew}
        \min_{\bar{\pi}} J_{r}(f,\pi^{*},\bar{\pi}) - \min_{\bar{\pi}}J_{r}(f,\pi_{t},\bar{\pi}) \leq 4L_{(r,\lambda,u)}\beta_{T}^{H}C^{H}H  \sum_{h^{'}=0}^{H} \mathbb{E}\left[\|\sigma_{t-1}\big(s_{h',t},\pi_t(s_{h',t}),\bar{\pi}_t(s_{h',t})\big)\|_{2}\right].
    \end{equation}
\end{lemma}
\begin{proof}
    Let us consider $\min_{\bar{\pi}} J(f,\pi^{*},\bar{\pi}) - \min_{\bar{\pi}}J(f,\pi_{t},\bar{\pi})$ as the instantaneous robust regret of the selected policy $\pi_{t}$ and the quantity $\min_{\bar{\pi}}J_{u}(f,\pi_{t},\bar{\pi})-b$ as the instantaneous robust violation of the same policy.
    We first determine the instantaneous regret bound, followed by the instantaneous violation bound.
        \begin{align}
            &\min_{\bar{\pi}} J_{r}(f,\pi^{*},\bar{\pi}) - \min_{\bar{\pi}} J_{r}(f,\pi_{t},\bar{\pi}) \leq \min_{\bar{\pi}} J_{r_t}^{(o)}(\pi^{*},\bar{\pi}) - \min_{\bar{\pi}} J(f,\pi_{t},\bar{\pi})\label{eq:51}\\
            &\leq \min_{\bar{\pi}} \left(J_{r_t}^{(o)}(\pi^{*},\bar{\pi})-\lambda[b-J_{u}^{(o)}(\pi^{*},\bar{\pi})]_{+}\right)- \min_{\bar{\pi}} J_{r}(f,\pi_{t},\bar{\pi})\label{eq:52}\\
            &\leq \min_{\bar{\pi}}\left(J_{r_t}^{(o)} (\pi_{t},\bar{\pi})-\lambda[b-J_{u}^{(o)}(\pi_{t},\bar{\pi})]_{+}\right)- \min_{\bar{\pi}} J_{r}(f,\pi_{t},\bar{\pi})\label{eq:53}\\
           &\leq \underbrace{J_{r_t}^{(o)}(\pi_{t},\bar{\pi}_{t})-\lambda[b-J_{u}^{(o)}(\pi_{t},\bar{\pi}_{t{}})]_{+}}_{t1} - \underbrace{\min_{\bar{\pi}} \left(J_{r}(f,\pi_{t},\bar{\pi})-\lambda[b-J_{u}(\pi_{t},\bar{\pi})]_{+}\right)}_{t2}\label{eq:54}\\
            &\leq \underbrace{J_{r_t}^{(o)}(\pi_{t},\bar{\pi}_{t})-\lambda[b-J_{u}^{(o)}(\pi_{t},\bar{\pi}_{t})]_{+}}_{t1} -  \underbrace{\min_{\bar{\pi}} \left(J_{r}^{(p)}(f,\pi_{t},\bar{\pi})-\lambda[b-J_{u}^{(p)}(\pi_{t},\bar{\pi})]_{+}\right)}_{t2}\label{eq:55}\\
            &\leq \underbrace{J_{r_t}^{(o)}(\pi_{t},\bar{\pi}_{t}) -\lambda[b-J_{u}^{(o)}(\pi_{t},\bar{\pi}_{t{}})]_{+}}_{t1} - \underbrace{J_{r}^{(p)}(\pi_{t},\bar{\pi}_{t})+\lambda[b-J_{u}^{(p)}(\pi_{t},\bar{\pi}_{t})]_{+}}_{t2}\label{eq:56}\\
            &\leq \underbrace{|J_{r_t}^{(o)}(\pi_{t},\bar{\pi}_{t}) - J_{r}(f,\pi_{t},\bar{\pi_{t}})|}_{t1}+\underbrace{|J_{r}^{(p)}(\pi_{t},\bar{\pi}) - J_{r}(f,\pi_{t},\bar{\pi}_{t})|}_{t2}+\underbrace{\lambda[J_{u}^{(o)}(\pi_{t},\bar{\pi}_{t})-J_{u}^{(p)}(\pi_{t},\bar{\pi}_{t})]_{+}}_{t3} \label{eq:57}\\
            &\leq \underbrace{|J_{r_t}^{(o)}(\pi_{t},\bar{\pi}_{t}) - J_{r}(f,\pi_{t},\bar{\pi_{t}})|}_{t1}+\underbrace{|J_{r}^{(p)}(\pi_{t},\bar{\pi}) - J_{r}(f,\pi_{t},\bar{\pi_{t}})|}_{t2}+\underbrace{\lambda|J_{u}^{(o)}(\pi_{t},\bar{\pi}_{t})-J_{u}^{(p)}(\pi_{t},\bar{\pi}_{t})|}_{t3}\label{eq:58} \\
            &\leq |J_{r_t}^{(o)}(\pi_{t},\bar{\pi}_{t}) - J_{r}(f,\pi_{t},\bar{\pi_{t}})|+|J_{r}^{(p)}(\pi_{t},\bar{\pi}) - J_{r}(f,\pi_{t},\bar{\pi_{t}})|\nonumber\\
            &+\lambda|J_{u}^{(p)}(\pi_{t},\bar{\pi}_{t})-J_{u}(f,\pi_{t},\bar{\pi_{t}})|+\lambda|J_{u}^{(o)}(\pi_{t},\bar{\pi}_{t})-J_{u}(f,\pi_{t},\bar{\pi_{t}})|\label{eq:60}\\
            &\leq L_{r,\pi} \sum_{h=0}^{H} \mathbb{E}\left[\|\mathbf{s}_{h,t} - \mathbf{s}_{h,t}^{(o)}\|_{2}+\|\mathbf{s}_{h,t} - \mathbf{s}_{h,t}^{(p)}\|_{2}\right]+\lambda L_{u,\pi} \sum_{h=0}^{H} \mathbb{E}\left[\|\mathbf{s}_{h,t} - \mathbf{s}_{h,t}^{(o)}\|_{2}+\|\mathbf{s}_{h,t} - \mathbf{s}_{h,t}^{(p)}\|_{2}\right]\label{proof:opt_pes_rew}
        \end{align}
    The first inequality (eq.~\eqref{eq:51}) comes from converting the the model function $f$ with the optimistic model function $f^{(o)}$ which is sure to return the maximum expected return for $\pi^{*}$. The second inequality (eq.~\eqref{eq:52}) is true as $\pi^{*}$ is the optimal policy so $[b-J_{u}^{(o)}(\pi^{*},\bar{\pi})] \leq 0$ then $\lambda[b-J_{u}^{(o)}(\pi^{*},\bar{\pi})]_{+} = 0$ so subtracting $0$ from the previous inequlaity (eq. \ref{eq:51}) will not affect the outcome. The third inequality (eq.~\eqref{eq:53}) comes by replacing the optimal policy by the protagonistic policy returned at time $t$. The fourth inequality (eq.~\eqref{eq:54}) is true since we remove the minimum $\min_{\bar{\pi}} J^{(o)}_{r}(.,\bar{\pi})$ with any other policy $\bar{\pi}_{t}$ whose return value is certain to be above the one returned by $\min_{\bar{\pi}}(J_{r_{t}}^{(o)}(.,\bar{\pi})+\lambda[b-J_{u}^{(o)}(.,\bar{\pi})]_{+})$. The fifth inequality (eq. \eqref{eq:55}) is obtained by transitioning the model function $f$ in the negative term (t2 in eq. \eqref{eq:54}) with a lower returning pessimistic path's $f^{(p)}$ value function. Now basically for the second term (t2) in the sixth inequality(eq. \eqref{eq:56}), we have our solution $\bar{\pi}_{t}$ returned by \texttt{RHC-UCRL} replacing the term $\bar{\pi}$. Now, in the seventh inequality(eq. \eqref{eq:57}), we combine the different terms, take their absolute values, and impose the inequality $x \leq |x|$ on terms t1 and t2. For term t3 in eq. \eqref{eq:57}, we consider the following inequality ($\max(a,0) - \max(b,0) \leq \max(a-b,0)$) (we prove this in proposition \ref{prop:inf_norm}) and then reducing it to simplest terms by cancelling $b$ . In the eigth inequality (eq. \eqref{eq:58}), we replace the $[.]_{+}$ with $|.|$ which follows the relation $[x]_{+} \leq |x|$\cite{shalev2014understanding}. The ninth inequality (eq. \eqref{eq:60} is an application of triangle inequality and in the tenth inequality(eq. \eqref{proof:opt_pes_rew}) we directly apply Lemma \ref{lem:reward_bind} and Lemma \ref{lem:utility_bind}.
    
    Then upon continuation of equation \eqref{proof:opt_pes_rew}, we use Lemma \ref{lem:state_bind}, it follows that all the terms inside the expectation are bounded in the same way as follows:
    \begin{equation}
    \begin{aligned}
        \|\mathbf{s}_{h,t} - \mathbf{s}_{h,t}^{(o)}\|_{2} \leq 2\beta_{t-1}\left(\left(1 + L_{f} + 2\beta_{t-1} L_{\sigma} \right)\sqrt{1 + L_\pi^2 + L_{\bar{\pi}}^2}\right)^{h}\\\sum_{h'=0}^{h}\left\|\sigma_{t-1}\big(s_{h',t},\pi_t(s_{h',t}),\bar{\pi}_t(s_{h',t})\big)\right\|_2,
    \end{aligned}
    \end{equation}
    as $f^{(o)}$ and $f^{(p)}$ belong to the same plausible models $\mathcal{M}_{t}$. By applying the above inequality twice and denoting $C:=\left( \left(1+L_{f} + 2L_{\sigma} \right)\sqrt{1 + L_\pi^2 + L_{\bar{\pi}}^2}\right)$, we arrive at
        \begin{align}
             \min_{\bar{\pi}} J_{r}(f,\pi^{*},\bar{\pi}) - \min_{\bar{\pi}} J_{r}(f,\pi_{t},\bar{\pi}) \leq 4(L_{r}+\lambda L_{u})\beta_{T}^{H}C^{H}\nonumber\\\sum_{h=0}^{H} \mathbb{E}\left[\sum_{h^{'}=0}^{h}\|\sigma_{t-1}\big(s_{h',t},\pi_t(s_{h',t}),\bar{\pi}_t(s_{h',t})\big)\|_{2}\right],\\
             =4L_{(r,\lambda,u)}\beta_{T}^{H}C^{H}\sum_{h=0}^{H} \mathbb{E}\left[\sum_{h^{'}=0}^{h}\|\sigma_{t-1}\big(s_{h',t},\pi_t(s_{h',t}),\bar{\pi}_t(s_{h',t})\big)\|_{2}\right],\\
             \leq 4L_{(r,\lambda,u)}\beta_{T}^{H}C^{H}H \sum_{h^{'}=0}^{H} \mathbb{E}\left[\|\sigma_{t-1}\big(s_{h',t},\pi_t(s_{h',t}),\bar{\pi}_t(s_{h',t})\big)\|_{2}\right],\label{eq:65}
        \end{align}
    where $t \leq T$ and $1 \leq \beta_{t}$ is non-decreasing in t. The last inequality (eq.\eqref{eq:65}) is obtained by taking the outer summation $\sum_{h=0}^{H}$ inside the expectation as expectation is a linear operator, and then performing the merging of the two summations into one by causing a change of variables, which results in an additional $(H-h+1)$ term inside the expectation. Now we consider the term $(H-h+1)$; as $h$ increases, it changes, so we replace each occurrence with the maximum value $H$ and take that common. 

\end{proof}

\begin{theorem*}
\label{thm_proof}
     Under Assumptions \ref{assume1} to \ref{assume3} and considering
    \[
    C=\left( \left(1+L_{f} + 2L_{\sigma} \right)\sqrt{1 + L_\pi^2 + L_{\bar{\pi}}^2}\right),\] $s_{h,t} \in \mathcal{S}$, $a_{h,t} \in \mathcal{A}$ and $\bar{a}_{h,t} \in \mathcal{\bar{A}}$ for all $t,h > 0$  and $\sigma_{t-1}^{h} = \sigma_{t-1}(s_{h,t},\pi_{t}(s_{h,t}),\bar{\pi}(s_{h,t}))$. Then for a fixed $H \geq 1$, with probability at least $1-\delta_{r}$, the robust cumulative regret of RHC-UCRL is upper bounded by:
   \[
R_T = \mathcal{O}\left( L(r,\lambda,u)\beta_T^H C^H H^{3/2} \sqrt{T \Gamma_T} \right),
\]
and with probability at least $1-\delta_u$,
\[
V_T = \mathcal{O}\left( L_u \beta_T^H C^H (1+\alpha) H^{3/2} \sqrt{T \Gamma_T} \right).
\]
\end{theorem*}

\begin{proof}
We first bound the cumulative regret and then the cumulative violation.
\begin{align}
&R_{T}= \sum_{t=1}^{T}
\underbrace{
\min_{\bar{\pi}\in\bar{\Pi}} J(f,\pi^{*},\bar{\pi})
-
\min_{\bar{\pi}\in\bar{\Pi}} J(f,\pi_{t},\bar{\pi})
}_{:=\,r_t} \\
&\leq \sqrt{T\sum_{t=0}^{T} r_t^2} \label{eqn:67}
\end{align}
\begin{align}
&\leq 4L_{(r,\lambda,u)}\beta_{T}^{H}C^{H}H\sqrt{T} 
\sqrt{
\sum_{t=1}^{T}
\left(
\mathbb{E}\!\left[
\sum_{h'=0}^{H}
\left\|
\sigma_{t-1}^{h^'}
\right\|_{2}
\right]
\right)^2
} \label{eqn:68}
\end{align}
\begin{align}
&\leq 4L_{(r,\lambda,u)}\beta_{T}^{H}C^{H}H\sqrt{T} 
\sqrt{
\sum_{t=1}^{T}
\mathbb{E}\!\left(
\left[
\sum_{h'=0}^{H}
\left\|
\sigma_{t-1}^{h^{'}}
\right\|_{2}
\right]^2
\right)
} \label{eqn:69}
\end{align}
\begin{align}
&\leq 4L_{(r,\lambda,u)}\beta_{T}^{H}C^{H}H^{1.5}\sqrt{T}
\sqrt{
\underbrace{
\sum_{t=1}^{T}
\mathbb{E}\!\left(
\left[
\sum_{h'=0}^{H}
\left\|
\sigma_{t-1}^{h^'}
\right\|_{2}^{2}
\right]
\right)
}_{\Gamma_{T}}}. \label{eqn:70}
\end{align}
where inequality \eqref{eqn:67} comes from Cauchy-Schwarz inequality; equation \eqref{eqn:68} is directly taken from our Lemma \ref{lem:opt_pes_rew_util} ; equation \eqref{eqn:69} is true due to application of Jensen's inequality; equation \eqref{eqn:70} is true due to application of Cauchy-Schwarz's inequality to equation \eqref{eqn:69} and then by definition of $\Gamma_{T}$, and the required proof follows

    Similarly, we now derive the same for the cumulative Violation bound:

\begin{align}
&V_{T}= \sum_{t=1}^{T}
\underbrace{
\big(b- J(f,\pi_{t},\bar{\pi}_{t})\big)_{+}
}_{v_{t}} \label{eqn:72}\\
 &\leq \sqrt{T\sum_{t=1}^{T}v_{t}^{2}}\label{eqn:73}\\ %
&\leq 2L_{u}H \beta_{T}^{H}C^{H}(1+\alpha)
\sqrt{T
\sum_{t=1}^{T}
\left(
\sum_{h=0}^{H}
\mathbb{E}\!\left[
\sum_{h'=0}^{H}
\left\|
\sigma_{t-1}^{h^{'}}
\right\|_{2}
\right]
\right)^2
} \label{eqn:74}\\
&\leq 2L_{u}H \beta_{T}^{H}C^{H}(1+\alpha) 
\sqrt{T
\sum_{t=1}^{T}
\mathbb{E}\!\left(
\left[
\sum_{h'=0}^{H}
\left\|
\sigma_{t-1}^{h^'}
\right\|_{2}
\right]^2
\right)
} \label{eqn:75}\\
&\leq 2L_{u}\beta_{T}^{H}C^{H}(1+\alpha)H^{1.5}\sqrt{T}
\sqrt{
\underbrace{
\sum_{t=1}^{T}
\mathbb{E}\!\left(
\left[
\sum_{h'=0}^{H}
\left\|
\sigma_{t-1}^{h^{'}}
\right\|_{2}^{2}
\right]
\right)
}_{\Gamma_{T}}}\label{eqn:76}
\end{align}
    where the first inequality(eq. \eqref{eqn:73}) is true due to the Cauchy-Schwarz inequality; the second inequality (eq. \eqref{eqn:74}) is due to Lemma \ref{lem:robust_viol_bound}, the third inequality (eq. \eqref{eqn:75}) is due to application of Jensen's inequality and the final inequality (eq. \eqref{eqn:76}) is due application of Cauchy-Schwarz's inequality. Finally, we use the definition of $\Gamma_{T}$ and the statement follows. This completes the required proof.
\end{proof}

\begin{prop}
\label{prop:inf_norm}
\[
[a]_{+} - [b]_{+} \leq [a-b]_{+}
\]
\end{prop}
\begin{proof}
    We know
    \begin{align}
        \|\mathbf{x}\|_{\infty} = \max{\left(\begin{bmatrix}
            x_{1}\\
            x_{2}\\
            x_{3}\\
            \ldots\\
            x_{d}
        \end{bmatrix}\right)}
    \end{align}
Now considering $d=2$ for any $a$ for which, we need to find $[a]_{+}$ let $\mathbf{a} = \begin{bmatrix}
    a\\
    0
\end{bmatrix}$, we have $[a]_{+} = \|\mathbf{a}\|_{\infty}$.

Then, 
\begin{align}
    &[a]_{+} = \|\mathbf{a}\|_{\infty}\\
    &\|\mathbf{a}\|_{\infty} = \|\mathbf{a} - \mathbf{b} + \mathbf{b}\|_{\infty} \leq \|\mathbf{a} - \mathbf{b}\|_{\infty}+\|\mathbf{b}\|_{\infty} \label{eq:triangle_inequality_inf_norm}\\
\end{align}
The inequality above (eq. \eqref{eq:triangle_inequality_inf_norm}) is true due to triangle inequality of infinity norm($\|.\|_{\infty}$). Thus, taking $\|\mathbf{b}\|_{2}$ on the other side, we get
\begin{align}
    \|\mathbf{a}\|_{\infty} - \|\mathbf{b}\|_{\infty} \leq \|\mathbf{a} - \mathbf{b}\|_{\infty}
\end{align}
\end{proof}

\begin{prop}
    \[
    [a]_{+} \leq |a|
    \]
\end{prop}
\begin{proof}
    \begin{align}
        [a]_{+} = \max{(a,0)},\\
    \end{align}
    We know that 
    \begin{align}
        a \leq |a|~\text{and}~ 0 \leq |a|
    \end{align}
    Thus combining both we can say $[a]_{+} \leq |a|$. This completes the proof
\end{proof}

\color{black}{}

\end{document}